\documentclass{article}
\usepackage{andydong_macros}
\usepackage[margin=1in]{geometry}

\title{Less Random, More Private: What is the Optimal Subsampling Scheme for DP-SGD?}

\author{
Andy Dong \\
Stanford University
\and
Ayfer \"Ozg\"ur \\
Stanford University
}

\date{May 2026}

\begin{document}

\maketitle

\begin{abstract}
Poisson subsampling is the default sampling scheme in differentially private machine learning, largely because its unstructured randomness yields tractable privacy amplification analyses. Yet this same randomness introduces substantial participation variance: each sample appears in very different numbers of training iterations. In this work, we show that this variance is not merely a practical artifact to be tolerated, but a fundamental source of suboptimal privacy amplification. We prove that Balanced Iteration Subsampling (BIS), a structured scheme in which each sample participates in exactly a fixed number of iterations, achieves stronger privacy amplification than Poisson subsampling and is optimal at both extremes of the noise spectrum ($\sigma \to 0$ and $\sigma \to \infty$). Our analysis reveals that the privacy-noise tradeoff is governed not by maximizing randomness, but by eliminating participation variance while preserving uniform marginal participation across iterations. To translate this asymptotic theory into finite-noise guarantees, we introduce a practical near-exact Monte Carlo accountant for BIS, which removes the analytical slack of existing RDP and composition-based PLD analyses. Evaluations across more than 60 practical DP-SGD configurations show that BIS consistently outperforms Poisson subsampling in the low-noise regimes most relevant for high-utility private training, reducing the required noise multiplier by up to $9.6\%$. These results overturn the common intuition that more sampling randomness necessarily yields stronger privacy amplification: in DP-SGD, structured participation can be both more practical and more private. Our implementation is available at \url{https://github.com/dong-xin-ao-andy/bis-mc-accountant}.
\end{abstract}

\section{Introduction}

Differential privacy (DP) provides a rigorous framework for protecting individual privacy in machine learning \cite{dwork2006calibrating}. In differentially private stochastic gradient descent (DP-SGD) \cite{abadi2016deep}, random minibatch sampling leads to privacy amplification, making the data subsampling scheme a central driver of the privacy-utility tradeoff.

For the past decade, Poisson subsampling has been the default choice for DP-SGD: each sample participates independently in each iteration with a fixed probability. This choice is motivated by analytical convenience and by the prevailing intuition that unstructured randomness should maximize privacy amplification. Yet the same unconstrained randomness creates undesirable variability in practice. Random batch sizes can complicate system-level implementation \cite{ganesh2025tighter}, while uneven participation counts across samples or users can raise fairness concerns and create unbalanced communication loads in federated learning \cite{dong2025leveraging, feldman2025privacy}.

In this work, we ask: \emph{What is the theoretically optimal data subsampling scheme for DP-SGD?} We study this question within the family of independent-example mechanisms, where each sample draws its participation pattern independently subject to a fixed expected participation count. We prove that, in both the low-noise limit ($\sigma \to 0$) and the high-noise limit ($\sigma \to \infty$), the privacy-optimal strategy is Balanced Iteration Subsampling (BIS) \cite{dong2025leveraging, feldman2025privacy}. In this structured scheme, each sample participates in exactly $k$ out of $T$ iterations, chosen uniformly at random, thereby eliminating participation-count variance across samples.

At first glance, it may seem counterintuitive that a more constrained sampling scheme can amplify privacy more strongly than unconstrained Poisson subsampling. Our analysis resolves this apparent paradox by uncovering the hierarchy of quantities that control the privacy loss. In the low-noise limit, the leading-order term is driven by fluctuations in how often a sample participates, making participation-count variance the primary obstacle to privacy amplification. In the high-noise limit, the leading-order term is governed by the marginal participation probabilities across iterations. BIS is optimal in both regimes because it eliminates participation-count variance while preserving uniform marginals across time.

While this asymptotic theory provides a clear conceptual explanation for the advantage of BIS, turning this insight into practical gains requires tightly evaluating its privacy amplification at finite noise levels $\sigma$. Exact $(\varepsilon,\delta)$ accounting for BIS is computationally intractable, and existing approaches have therefore relied on relaxations through R\'enyi differential privacy (RDP) \cite{dong2025leveraging, mironov2017renyi} or composition-based privacy loss distributions (PLDs) \cite{feldman2026efficient, doroshenko2022connect}. These reductions introduce analytical slack, obscuring the true privacy amplification of BIS and masking its advantage over Poisson subsampling.

To bridge the gap between asymptotic theory and practical verification, we introduce a near-exact Monte Carlo accountant for BIS. We bypass the combinatorial explosion by designing a dynamic programming formulation that reduces exact likelihood-ratio evaluation from $\mc{O}\big(\binom{T}{k}\big)$ to $\mc{O}(Tk)$ time. We further combine this with a novel $\mc{O}(T)$ screening upper bound that serves as a computational filter: it immediately discards the vast majority of sampled realizations and invokes the exact computation only on the remaining $\mc{O}(\delta)$ fraction of samples.

This algorithm allows us to remove the slack of prior bounds and reveal the true privacy-noise profile of BIS. We show that BIS consistently outperforms Poisson subsampling in low-noise regimes, the practical setting most relevant to utility and deployment, yielding up to a 9.6\% reduction in the required noise multiplier. In the high-noise limit, BIS matches the performance of Poisson subsampling.

In summary, our contributions are threefold:
\begin{itemize}
\item \textbf{Theoretical optimality:} We prove that BIS is optimal at both extremes of the noise spectrum ($\sigma \to 0$ and $\sigma \to \infty$) among independent-example subsampling schemes.

\item \textbf{Near-exact finite-noise accounting:} We introduce a practical Monte Carlo accountant for BIS, powered by an $\mc{O}(Tk)$ dynamic program and an $\mc{O}(T)$ screening bound, which overcomes the combinatorial intractability of the $k$-out-of-$T$ sampling scheme and removes the slack in prior bounds.

\item \textbf{Empirical privacy gains:} By eliminating the analytical looseness of RDP and composition-based PLD methods, our near-exact evaluation shows that BIS is a principled better alternative to Poisson subsampling across realistic DP-SGD configurations.
\end{itemize}

\paragraph{Code Availability.} The complete Python implementation of our Monte Carlo accountant is publicly available at \url{https://github.com/dong-xin-ao-andy/bis-mc-accountant}.

\section{Background}

\subsection{Differential Privacy}

Let $\mc{X}$ denote the data domain and let $\mc{M}$ be a randomized algorithm mapping datasets to outputs. We consider the zero-out adjacency relation. Two datasets $D, D' \subseteq \mc{X}$ are adjacent if $D'$ can be obtained from $D$ by replacing one element of $D$ with a special record $\perp$ that always contributes zero gradient.

Let $P$ and $Q$ denote the distributions of $\mc{M}(D)$ and $\mc{M}(D')$, respectively. The mechanism $\mc{M}$ satisfies $(\varepsilon, \delta)$-DP if the $e^\epsilon$-hockey-stick divergence between $P$ and $Q$ is bounded by $\delta$, i.e.,
\begin{equation}
\label{eq:DP_definition}
H_{e^\varepsilon}(P \;\Vert\; Q)
= \mathbb{E}_{y \sim P} \left[ \max \left\{ 1 - e^\varepsilon \frac{Q(y)}{P(y)}, \: 0 \right\} \right ]
= \mathbb{E}_{y \sim Q} \left[ \max \left\{ \frac{P(y)}{Q(y)} - e^\varepsilon, \: 0 \right\} \right] \leq \delta.
\end{equation}

\subsection{Monte Carlo Privacy Accounting}
\label{sec:background_mc_accounting}

Monte Carlo accounting is a technique in privacy analysis that certifies guarantees by sampling and constructing a statistical confidence interval \cite{chua2024balls}. Suppose a mechanism $\mc{M}$ produces outputs distributed as $P$ on dataset $D$ and as $Q$ on a neighboring dataset $D'$. Computing their hockey-stick divergence analytically is often intractable for complex subsampling mechanisms. Monte Carlo accounting instead draws many outputs from one of the distributions and empirically evaluates the likelihood ratio between the two. We can obtain an unbiased estimate of \eqref{eq:DP_definition} by drawing samples $y_1,\ldots,y_s \sim P$ and computing
\begin{equation}
\hat{\delta} = \frac{1}{s} \sum_{i=1}^{s} \max\left\{1 - e^\varepsilon \frac{Q(y_i)}{P(y_i)},\,0\right\}.
\end{equation}

When the likelihood ratio $P(y)/Q(y)$ (or its inverse) can be evaluated efficiently for all $y$, this estimator provides a highly practical method for computing near-exact privacy guarantees. The fundamental computational challenge is therefore the efficient evaluation of this likelihood ratio.

It is important to note that the Monte Carlo estimation alone does not directly provide a formal privacy guarantee because the empirical estimate $\hat{\delta}$ is a random variable. To address this issue and provide rigorous bounds, prior work proposed the Estimate–Verify–Release (EVR) framework \cite{wang2023randomized}. EVR uses Monte Carlo estimation as a boolean verification step for a candidate privacy parameter and derives a valid, end-to-end guarantee by absorbing the failure probability of the statistical verifier into the overall $\delta$ budget.

Furthermore, a recent modification of this framework removes the possibility of returning an empty output when the mechanism does not pass the verifier, allowing the noise parameter to be tuned jointly with the verification step \cite{dong2026privacy}. In our evaluation, we adopt this modified EVR Monte Carlo accounting framework and focus entirely on solving the core computational bottleneck, which is efficiently evaluating the exact likelihood ratio for Balanced Iteration Subsampling so that the near-exact Monte Carlo guarantees can be deployed.

We note that although Monte Carlo verification requires sampling many realizations, it does not present a practical bottleneck for large-scale training. The number of samples required is typically $\mc{O}(\log(1/\delta) / \delta)$, where $1 / \delta$ is on the same order as the dataset size. Recent scaling laws for DP-SGD suggest that the compute required to saturate model utility grows quadratically with dataset size \cite{mckenna2025scaling}. Consequently, at scale where one may expect a large accounting overhead, the relative overhead actually shrinks compared to the total training cost. Furthermore, because MC accounting is parallelizable, it executes highly efficiently on readily available CPUs and does not compete for the scarce accelerator resources (GPUs/TPUs) required for the actual model training.

\section{Optimal Subsampling for Privacy Amplification}
\label{sec:optimality}

To maximize privacy amplification in DP-SGD, we analyze the subsampling mechanism from first principles. To ensure a rigorous search for the optimal scheme, we first define the design space.

\begin{definition}[Independent-Example Subsampling Family]
\label{def:indep_example}
Let a training procedure consist of $T$ total iterations, and let $k \in (0, T)$ be the expected participation count. We define $\mc{P}_{T,k}$ as the family of all probability distributions $P$ over the Boolean hypercube $\{0, 1\}^T$ satisfying the expected participation count constraint
\begin{equation}
\mathbb{E}_{x \sim P} \|x\|_2^2 = \sum_{x \in \{0,1\}^T} P(x) \|x\|_2^2 = k.
\end{equation}
We define the \textit{independent-example subsampling schemes} as the set of batch-construction mechanisms where the participation vector $x^{(i)}$ for each example $i$ is drawn independently and identically from some $P \in \mc{P}_{T,k}$.
\end{definition}

This family captures all mechanisms where examples participate independently from each other with a fixed expected count (e.g., Poisson subsampling, where each iteration is selected independently with probability $k/T$). Since the differing example introduces exactly one independent participation vector $x \sim P$, the mechanism's worst-case privacy loss is uniquely characterized by $P$. Our goal is to find the optimal $P \in \mc{P}_{T,k}$ that maximizes utility or minimizes the $(\varepsilon, \delta)$ privacy leakage.

We fix the expected participation count $k$ to isolate the structural impact of the sampling mechanism. As observed in prior work \cite{dong2025leveraging, dong2026privacy}, when the expected participation count and noise level are held constant, structurally symmetric and well-distributed participation patterns (such as Poisson and Balanced Iteration Subsampling) yield indistinguishable model utility and mean squared error in the gradient prefix sum. Thus, by restricting our search to the family $\mc{P}_{T,k}$ (with a focus on schemes that do not arbitrarily concentrate participations), we can decouple utility from our analysis and optimize for the lowest noise multiplier.

To evaluate the privacy guarantees of a candidate mechanism $P \in \mc{P}_{T,k}$, we analyze its privacy loss random variable. For a realized mechanism output $y$, the privacy loss is defined as the log-likelihood ratio $L(y) = \log(P(y)/Q(y))$. Under the adjacent dataset lacking the differing example, the output is purely Gaussian noise, yielding the density
\[Q(y) \propto \exp \left(-\frac{\norm{y}_2^2}{2\sigma^2} \right).\]
Under the dataset containing the differing example, the output is a Gaussian mixture over all possible participation vectors $\tilde{x} \sim P$, giving
\[P(y) \propto \sum_{\tilde{x}} P(\tilde{x}) \exp \left( - \frac{\norm{y - \tilde{x}}_2^2}{ 2\sigma^2} \right).\]
We note that this reduction to multivariate Gaussian mechanisms is standard in DP accounting analysis. This represents the worst-case scenario for privacy loss under zero-out adjacency, as any gradient with a smaller norm or varying direction would have a privacy profile dominated by this worst case \cite{choquette2023privacy}.

Expanding the squared Euclidean norm in $P(y)$ yields
\[\norm{y - \tilde{x}}_2^2 = \norm{y}_2^2 - 2\inner{\tilde{x}, y} + \norm{\tilde{x}}_2^2.\]
When taking the ratio $P(y)/Q(y)$, the common $\norm{y}_2^2$ factors cancel. The log-likelihood ratio, evaluated at an arbitrary output $y$, can therefore be written as
\begin{equation}
\label{eq:l_y_full}
L(y) = \log\left( \E_{\tilde{x} \sim P}\left[ \exp\left( \frac{2\inner{\tilde{x}, y} - \norm{\tilde{x}}_2^2}{2\sigma^2} \right) \right] \right).
\end{equation}

\subsection{Deriving the Low-Noise Optimum (\texorpdfstring{$\sigma \to 0$}{sigma -> 0})}
\label{sec:low_noise_optimality}

With the exact privacy loss established, we first consider the small-noise regime ($\sigma \to 0$), which corresponds to high utility and large privacy budgets—the regime most relevant to practical model deployment.

To evaluate the $(\varepsilon, \delta)$-DP guarantee, we use the hockey-stick divergence identity in \eqref{eq:DP_definition} where the expectation is over $y \sim P$. Under $P$, the output is generated as $y = x + \sigma w$, where the true participation vector $x \sim P$ and the standard normal noise $w \sim \Normal(0, I_T)$ is independent of $\sigma$. As $\sigma \to 0$, we can isolate the dominant term of the mixture in \eqref{eq:l_y_full}.

\begin{proposition}[Low-Noise Privacy Loss]
\label{prop:low_noise_loss}
As $\sigma \to 0$, the privacy loss random variable $L(y)$ evaluated at $y = x + \sigma w$ is asymptotically dominated by the term corresponding to the true sampled vector $x$, yielding
\begin{equation}
\label{eq:privacy_loss_low_noise}
L(y) \longrightarrow \frac{\norm{x}_2^2}{2\sigma^2} + \frac{\inner{x, w}}{\sigma} + \log P(x).
\end{equation}
\end{proposition}

\begin{proof}[Proof of Proposition~\ref{prop:low_noise_loss}]
To understand which term dominates the sum in \eqref{eq:l_y_full}, we explicitly write the gap between the exponent evaluated at the true sampled vector $x$ and the exponent evaluated at any false candidate $\tilde{x} \neq x$,
\begin{align}
&\phantom{{}={}} \left( \frac{2 \inner{x, x} - \norm{x}_2^2}{2\sigma^2} + \frac{\inner{x, w}}{\sigma} \right) - \left( \frac{2\inner{\tilde{x}, x} - \norm{\tilde{x}}_2^2}{2\sigma^2} + \frac{\inner{\tilde{x}, w}}{\sigma} \right) \nonumber \\
&= \left( \frac{\norm{x}_2^2}{2\sigma^2} + \frac{\inner{x, w}}{\sigma} \right) - \left( \frac{2\inner{\tilde{x}, x} - \norm{\tilde{x}}_2^2}{2\sigma^2} + \frac{\inner{\tilde{x}, w}}{\sigma} \right) \nonumber \\
&= \frac{\norm{x}_2^2 - 2\inner{x, \tilde{x}} + \norm{\tilde{x}}_2^2}{2\sigma^2} + \frac{\inner{x - \tilde{x}, w}}{\sigma} \nonumber \\
&= \frac{\norm{x - \tilde{x}}_2^2}{2\sigma^2} + \frac{\inner{x - \tilde{x}, w}}{\sigma}.
\end{align}
Because $x$ and $\tilde{x}$ are discrete binary vectors and $\tilde{x} \neq x$, the squared Euclidean distance satisfies $\norm{x - \tilde{x}}_2^2 \ge 1$. Consequently, as $\sigma \to 0$, the deterministic $\Theta(1/\sigma^2)$ term dominates the stochastic $\Theta(1/\sigma)$ noise term, and Laplace's principle (from large deviations theory) establishes that the logarithm of a sum of exponentials is asymptotically equivalent to the maximum exponent \cite{dembo2009large}. Because the gap between the exponent for the true mode $\tilde{x} = x$ and any false candidate $\tilde{x} \neq x$ diverges to $+\infty$, the log-sum-exp in \eqref{eq:l_y_full} is entirely dominated by the true sampled vector. Factoring out this dominant mode, we get a simple expression for the privacy loss in the limit $\sigma \to 0$,
\begin{equation}
L(y) \longrightarrow \frac{\norm{x}_2^2}{2\sigma^2} + \frac{\inner{x, w}}{\sigma} + \log P(x).
\end{equation}
\end{proof}

To understand how this dictates the optimal subsampling scheme, we connect the privacy loss $L(y)$ to the hockey-stick divergence in \eqref{eq:DP_definition}. For a given $\varepsilon$, this is governed by the upper tail of the privacy loss distribution. Therefore, minimizing $\delta(\varepsilon)$ for fixed $\varepsilon$ amounts to suppressing the probability that $L(y)$ takes large values.

The asymptotic expression for $L(y)$ in \eqref{eq:privacy_loss_low_noise} reveals three terms of decreasing order as $\sigma \to 0$: a $\Theta(1/\sigma^2)$ term driven by $\norm{x}_2^2$, a $\Theta(1/\sigma)$ fluctuation term driven by $\inner{x, w}$, and a $\Theta(1)$ constant term given by the log-probability $\log P(x)$ assigned to each participation vector $x$. This decomposition suggests a natural hierarchy for controlling the upper tail of $L(y)$: first minimizing the leading term, followed by the second, and then the third.

\textbf{First Layer of Optimization: Eliminating Participation Variance.} To minimize the first term, note that its expectation is fixed at $k / (2\sigma^2)$ for the class of subsampling schemes under consideration (Definition~\ref{def:indep_example}). Hence, the upper tail of $L(y)$ is minimized when this term exhibits no variation around its mean. This can be achieved by enforcing $\norm{x}_2^2 = k$ with probability one.\footnote{We assume $k \in \mathbb{Z}$. For non-integer $k$, variance is rigorously minimized by sampling from $\{\lfloor k \rfloor, \lceil k \rceil\}$, though simply selecting an integer $k$ is a practical alternative with negligible impact on training.} Any variance in $\norm{x}_2^2$ induces fluctuations of order $\Theta(1/\sigma^2)$, which dominate both the $\Theta(1/\sigma)$ fluctuations from the second term and the $\Theta(1)$ log-probability term.

\textbf{Second Layer of Optimization: Invariance of First-Order Noise.} Next, consider minimizing the contribution of $\inner{x, w}$ among all distributions $P$ satisfying $\norm{x}_2^2 = k$. For any such $x$, the inner product $\inner{x, w}$ is a zero-mean Gaussian random variable with variance $k$. Therefore, the choice of $P$ does not affect the distribution of this term.

\textbf{Third Layer of Optimization: Maximizing the Entropy.} Finally, consider the log-probability term $\log P(x)$. Its expectation is given by $\E_{x \sim P}[\log P(x)] = -H(P)$, the entropy of the subsampling scheme. Thus, controlling the upper tail of $\log P(x)$ amounts to decreasing its mean and reducing its variability. Both objectives are achieved by maximizing the entropy of $P$. This is attained by choosing $P$ to be the uniform distribution over all $\binom{T}{k}$ binary vectors with $k$ ones and $T-k$ zeros. This choice simultaneously maximizes $H(P)$ among schemes with fixed participation $\norm{x}_2^2 = k$, and eliminates variation in $\log P(x)$. In contrast, more structured subsampling schemes reduce entropy and are therefore suboptimal. For example, enforcing exactly one participation within each block of $T/k$ iterations restricts the support to $(T/k)^k$ configurations. Since $(T/k)^k \ll \binom{T}{k}$, such constraints significantly reduce randomness, as captured by $H(P)$, leading to a weaker privacy guarantee.

This argument helps clarify a potential misconception. One might a priori expect that maximizing the randomness in the subsampling scheme would lead to maximal amplification, an intuition that emphasizes maximizing $H(P)$. The analysis above, however, reveals a clear hierarchy: in the low-noise limit, controlling the upper tail of the privacy loss through the variability of $\norm{x}_2^2$ is significantly more important than maximizing $H(P)$, which only governs the third-order term. In particular, Poisson subsampling trades a potentially significant $\Theta(1/\sigma^2)$ deviation (arising from the right tail of the binomial distribution of $\norm{x}_2^2$) for a more modest $\Theta(1)$ gain in entropy. From this perspective, achieving optimality requires going beyond unconstrained randomness and instead carefully controlling participation variance.

\subsection{Formalizing Balanced Iteration Subsampling}

The optimal distribution we have just derived---forcing exactly $k$ participations and choosing among all valid combinations uniformly at random---corresponds exactly to a mechanism recently proposed in the literature under the names Balanced Iteration Subsampling \cite{dong2025leveraging} and Random Allocation \cite{feldman2025privacy}. 

Because it is derived as the theoretical optimum for the utility-favorable $\sigma \to 0$ regime, we adopt this mechanism for our primary analysis. Formally, for each example $i \in [n]$, we independently choose a subset $A_i \subseteq [T]$ of size $k$ uniformly at random. The example $i$ is included in iteration $t$ if and only if $t \in A_i$. For each iteration $t$, the batch used by the training algorithm is defined as $S_t = \{\, i \in [n] \st t \in A_i \,\}$. Algorithm~\ref{alg:bis} provides the pseudocode for implementing BIS. 

\begin{algorithm}[htbp]
\caption{Balanced Iteration Subsampling}
\label{alg:bis}
\begin{algorithmic}[1]
\State \textbf{Input:} number of iterations $T$, participation count $k$, dataset indices $[n]$
\State \textbf{Output:} batches $S_1,\dots,S_T$
\State $S_t \leftarrow \emptyset$ for all $t \in [T]$
\For{$i = 1,\dots,n$}
    \State sample $A_i \subseteq [T]$ uniformly at random with $\abs{A_i} = k$
    \For{each $t \in A_i$}
        \State $S_t \leftarrow S_t \cup \{i\}$
    \EndFor
\EndFor
\State \Return $S_1,\dots,S_T$
\end{algorithmic}
\end{algorithm}

This procedure satisfies the conditions of $\mc{P}_{T,k}$. When $k = 1$, the mechanism reduces to the classic balls-and-bins (or 1-out-of-$T$) sampling model \cite{chua2024balls}. Because privacy accounting for this $k = 1$ regime has been extensively studied in prior work \cite{chua2024balls, feldman2026efficient}, our analysis focuses on the generalized case where $k > 1$.

\subsection{Optimality in the High-Noise Limit (\texorpdfstring{$\sigma \to \infty$}{sigma -> infinity})}
\label{sec:high_noise_optimality}

We now consider the high-noise limit ($\sigma \to \infty$) using the hockey-stick divergence identity in \eqref{eq:DP_definition} where the expectation is over $y \sim Q$. Under $Q$, the output is generated as $y = \sigma w$ where $w \sim \Normal(0, I_T)$ is independent of $\sigma$.

\begin{proposition}[High-Noise Privacy Loss]
\label{prop:high_noise_loss}
Let $\mu = \E_{x \sim P}[x]$ denote the expected participation vector, and $\Sigma = \E[xx^T] - \mu\mu^T$ denote its covariance matrix. As $\sigma \to \infty$, the privacy loss $L(y)$ expands as
\begin{equation}
\label{eq:asymptotic_L}
L(y) = \frac{\inner{\mu, w}}{\sigma} + \frac{w^T \Sigma w}{2\sigma^2} - \frac{k}{2\sigma^2} + \mc{O}(1/\sigma^3).
\end{equation}
\end{proposition}

\begin{proof}[Proof of Proposition~\ref{prop:high_noise_loss}]
We substitute $y = \sigma w$ into \eqref{eq:l_y_full} to get
\begin{equation}
\frac{P(y)}{Q(y)} = \E_{x \sim P}\left[ \exp\left( \frac{2\inner{x, \sigma w} - \norm{x}_2^2}{2\sigma^2} \right) \right] = \E_{x \sim P}\left[ \exp\left( \frac{\inner{x, w}}{\sigma} - \frac{\norm{x}_2^2}{2\sigma^2} \right) \right].
\end{equation}
In the limit as $\sigma \to \infty$, we take the Taylor expansion of the exponential function, using $\exp(u) = 1 + u + u^2/2 + \mc{O}(u^3)$,
\begin{equation}
\frac{P(y)}{Q(y)} = \E_{x \sim P}\left[ 1 + \left( \frac{\inner{x, w}}{\sigma} - \frac{\norm{x}_2^2}{2\sigma^2} \right) + \frac{1}{2}\left( \frac{\inner{x, w}}{\sigma} \right)^2 + \mc{O}(1/\sigma^3) \right].
\end{equation}
Let $\mu = \E_{x \sim P}[x]$ denote the expected participation vector. Pushing the expectation through the linear terms and using the fact that $\E[\norm{x}_2^2] = k$ for all schemes in $\mc{P}_{T,k}$, we obtain
\begin{equation}
\frac{P(y)}{Q(y)} = 1 + \frac{\inner{\mu, w}}{\sigma} + \frac{\E[\inner{x, w}^2] - k}{2\sigma^2} + \mc{O}(1/\sigma^3).
\end{equation}
To isolate the privacy loss random variable $L(y) = \log(P(y)/Q(y))$, we apply the logarithmic expansion $\log(1 + u) = u - u^2/2 + \mc{O}(u^3)$ where $u = P(y)/Q(y) - 1$. Substituting this yields
\begin{equation}
L(y) = \frac{\inner{\mu, w}}{\sigma} + \frac{\E[\inner{x, w}^2] - \inner{\mu, w}^2}{2\sigma^2} - \frac{k}{2\sigma^2} + \mc{O}(1/\sigma^3).
\end{equation}
Using the covariance matrix of the participation vector $\Sigma = \E[xx^T] - \mu\mu^T$, we can rewrite the numerator of the second-order term as a quadratic form:
\begin{equation}
\E[\inner{x, w}^2] - \inner{\mu, w}^2 = w^T \E[xx^T] w - w^T \mu\mu^T w = w^T \Sigma w.
\end{equation}
Substituting this back yields the final asymptotic expression:
\begin{equation}
L(y) = \frac{\inner{\mu, w}}{\sigma} + \frac{w^T \Sigma w}{2\sigma^2} - \frac{k}{2\sigma^2} + \mc{O}(1/\sigma^3).
\end{equation}
\end{proof}

This asymptotic expression reveals a clear, different hierarchy for controlling the upper tail of $L(y)$, mirroring the layered optimization approach of the low-noise regime. To minimize the right tail of $L(y)$---which dictates the hockey-stick divergence $\delta(\varepsilon)$---we must first minimize the tail of the dominant $\Theta(1/\sigma)$ stochastic term, followed by the $\Theta(1/\sigma^2)$ second-order term.

\textbf{First Layer of Optimization: Uniform Marginals.} 
In the limit $\sigma \to \infty$, the dominant term in \eqref{eq:asymptotic_L} is $\inner{\mu, w} / \sigma$. Because $w \sim \Normal(0, I_T)$, this term is distributed as $\Normal(0, \norm{\mu}_2^2 / \sigma^2)$. To suppress its right tail, we simply minimize $\norm{\mu}_2^2$. 

Because $x \in \{0, 1\}^T$, the constraint on the expected participation count (Definition~\ref{def:indep_example}) dictates that $\sum_{t=1}^T \mu_t = k$. By the Cauchy-Schwarz inequality,
\begin{equation}
\norm{\mu}_2^2 \ge \frac{1}{T} \left( \sum_{t=1}^T \mu_t \right)^2 = \frac{k^2}{T}.
\end{equation}
Equality holds---and the first-order tail is minimized---if and only if $\mu_t = k/T$ for all $t \in [T]$. Thus, in the $\sigma \to \infty$ regime, any optimal subsampling scheme must have uniform marginal participation probabilities across all iterations. Both Poisson and Balanced Iteration Subsampling fall into this category.

\textbf{Second Layer: Eliminating Participation Variance (Again).}
Restricting our focus to the optimal class of uniform-marginal schemes (where $\mu = p\ones$ and $p = k/T$), we now analyze the $\Theta(1/\sigma^2)$ term in \eqref{eq:asymptotic_L}. To understand how these second-order fluctuations interact with the dominant first-order term, we orthogonally decompose the standard normal noise vector as $w = z \frac{\ones}{\sqrt{T}} + w_{\perp}$, where $z = \frac{1}{\sqrt{T}} \inner{\ones, w} \sim \Normal(0, 1)$ represents the magnitude of the noise along $\ones$, and $w_{\perp}$ lies in the orthogonal subspace and is independent of $z$.

Substituting this decomposition and $\mu = p\ones$ into the first-order term in \eqref{eq:asymptotic_L} yields $\frac{\inner{\mu, w}}{\sigma} = \frac{p \sqrt{T}}{\sigma} z$. The hockey-stick divergence is governed only by the tail where $L(y) > \varepsilon$. In the high-noise limit, the first-order term determines the magnitude of $L(y)$, meaning this tail event requires $\frac{p \sqrt{T}}{\sigma} z \gtrsim \varepsilon$. For the mechanism to achieve a fixed cryptographically small privacy target $\delta$ (typically $10^{-10}$ to $10^{-5}$), $z$ must reach a critical threshold $z^* \approx \Phi^{-1}(1-\delta)$, which is a positive constant bounded away from $0$ (e.g., $z^* \approx 4.26$ for $\delta = 10^{-5}$).

We also substitute the decomposition of $w$ into the quadratic form of the second-order term in \eqref{eq:asymptotic_L}, which expands into three components:
\begin{equation}
\label{eq:second_order_term_decomp}
\frac{w^T \Sigma w}{2\sigma^2} = \frac{1}{2\sigma^2} \left( z^2 \frac{\ones^T \Sigma \ones}{T} + 2z \frac{\ones^T \Sigma w_{\perp}}{\sqrt{T}} + w_{\perp}^T \Sigma w_{\perp} \right),
\end{equation}
which we analyze in this critical tail regime where $z \gtrsim z^*$. The total variance is fixed at $\trace(\Sigma) = Tp(1-p)$ for any uniform-marginal scheme. Because we are operating in the tail where $z \gtrsim z^* \ge 4.26$, the coefficients $z^2$ and $z$ heavily amplify any variance or covariance associated with the $\ones$ direction in \eqref{eq:second_order_term_decomp}. In contrast, variance allocated to any orthogonal direction ($w_{\perp}$) is unscaled by $z$. To suppress the dominant worst-case tail, an optimal scheme must shift all variance out of the $\ones$ direction by setting $\ones^T \Sigma \ones = \var \big(\sum_{t=1}^T x_t \big) = 0$. By the positive semi-definiteness of $\Sigma$, this guarantees $\Sigma \ones = \zero$, which nullifies both the amplified parallel and cross components simultaneously. This zero-variance condition is uniquely achieved if and only if the participation count is deterministically fixed to $k$.

\textbf{Third Layer: Enforcing Pairwise Symmetry.}
Once the participation variance is eliminated ($\ones^T \Sigma \ones = 0 \implies \Sigma \ones = \zero$), the only remaining fluctuation in \eqref{eq:second_order_term_decomp} is the orthogonal component $\frac{1}{2\sigma^2} w_{\perp}^T \Sigma w_{\perp}$. Since the $\ones$ direction consumes none of the variance, the entirety of the fixed $\trace(\Sigma) = Tp(1-p)$ budget is distributed across the remaining $T-1$ orthogonal dimensions. By the Laurent--Massart bound \cite{laurent2000adaptive}, the upper tail of this remaining quadratic form is governed by the $\ell_2$ and $\ell_\infty$ norms of its eigenvalues. Suppressing this worst-case tail requires simultaneously minimizing both norms subject to the fixed trace constraint. This is uniquely achieved when the variance is distributed equally across all non-zero orthogonal directions, forcing the optimal covariance matrix to be isotropic over the subspace orthogonal to $\ones$, thus giving $\Sigma = \frac{Tp(1-p)}{T-1} \left( I_T - \frac{1}{T}\ones\ones^T \right)$.

Balanced Iteration Subsampling achieves this optimum as its combinatorial symmetry treats all participation patterns identically. In contrast, artificial designs, such as placing one participation in every block of $T/k$ iterations, inflate these eigenvalue norms, yielding a heavier privacy loss tail.

The limiting benchmark for these mechanisms can be understood as a scaled full-batch update, where the gradient sum is multiplied by $k/T$ to match the expected signal of the subsampled mechanisms. This is mathematically equivalent to relaxing the binary constraint $x \in \{0, 1\}^T$ and replacing the random participation vector with a fractional, deterministic vector $x = \frac{k}{T}\mathbf{1}$. In this deterministic case, the covariance matrix vanishes ($\Sigma = 0$), eliminating the second-order fluctuations in the privacy loss. This zero-variance configuration represents the fundamental full-batch lower bound approached by uniform-marginal subsampling schemes as $\sigma \to \infty$.

This spectral analysis also explains why the privacy guarantees of subsampled mechanisms converge toward those of this scaled full-batch update in the high-noise limit, a phenomenon observed empirically in prior work \cite{wang2019subsampled, dong2022gaussian, ponomareva2023dp}. As $\sigma \to \infty$, the first-order term is the dominant driver of privacy loss, and any scheme with uniform marginals, including both Poisson subsampling and BIS, matches this term. The advantage of BIS appears in the second-order term: by fixing the participation count, it removes fluctuations in the all-ones direction and thereby minimizes the remaining worst-case tail among uniform-marginal schemes. However, because this second-order advantage is asymptotically dominated by the shared first-order term, BIS and Poisson become numerically indistinguishable in the high-noise limit, both matching the fundamental lower bound.

\subsection{When Should BIS Improve Over Poisson in Practice?}
\label{sec:theoretical_comparison}

The two asymptotic regimes lead to a simple practical prediction. In the low-noise regime ($\sigma \to 0$), privacy loss is dominated by fluctuations in the total number of times a sample participates. BIS eliminates this participation-count variance, while Poisson subsampling leaves it unconstrained. We therefore expect BIS to exhibit a substantial privacy advantage over Poisson subsampling, with the gap becoming more pronounced as $\sigma$ decreases.

In the high-noise regime ($\sigma \to \infty$), the dominant term is instead determined by the marginal participation probabilities across iterations. Since both BIS and Poisson subsampling have uniform marginals, they match at the leading order. Although BIS remains theoretically optimal by minimizing the second-order fluctuations, this advantage becomes numerically negligible relative to the shared first-order term. Thus, the privacy guarantees of BIS and Poisson should converge in the high-noise limit, approaching the scaled full-batch lower bound.

Combining these two limits, BIS is optimal at both ends of the noise spectrum within the family of independent-example mechanisms (Definition~\ref{def:indep_example}). More importantly, the theory gives a concrete empirical prediction: BIS should provide a visible privacy gain over Poisson subsampling in low-noise, high-utility regimes, while the gap should narrow and vanish as the noise level increases. We verify this trajectory in Section~\ref{sec:evaluation}.

\section{Near-Exact Analysis of BIS via Monte Carlo Accounting}
\label{sec:accounting}

To empirically observe the theoretical optimality derived in Section~\ref{sec:optimality}, we require a privacy accountant with minimal analytical slack. Existing privacy analyses for Balanced Iteration Subsampling rely on R\'enyi differential privacy (RDP) or privacy loss distributions (PLDs) via composition bounds \cite{dong2025leveraging, feldman2025privacy, feldman2026efficient}. Because these methods pay the cost of translating RDP to $(\varepsilon, \delta)$-DP or bound the joint mechanism by composing intermediate guarantees, they can introduce analytical looseness that may obscure the tight privacy cost.

We therefore turn to Monte Carlo accounting to provide near-exact privacy guarantees. As established in Section~\ref{sec:background_mc_accounting}, the primary computational bottleneck in Monte Carlo accounting is evaluating the exact likelihood ratio $P(y)/Q(y)$ for any given $y$. Under Balanced Iteration Subsampling, a naive evaluation of this ratio requires computing all $\binom{T}{k}$ valid combinations, which is intractable for realistic training horizons.

In this section, we develop an efficient computational paradigm that makes near-exact Monte Carlo accounting practical for Balanced Iteration Subsampling. We achieve this by combining an $\mc{O}(Tk)$ dynamic program with an ultra-fast $\mc{O}(T)$ screening bound, effectively shifting the computational complexity from $ \mc{O}\big(\binom{T}{k}\big)$ to a time of $\mc{O}(T + \delta Tk)$ per sample heuristically.

\subsection{Exact Likelihood Ratio via Dynamic Programming}

We consider the likelihood ratio between the two neighboring datasets under zero-out adjacency. On the dataset where the differing example is absent, the mechanism output is purely noise,
$$ Q = \mc{N}(0,\sigma^2 I_T). $$
On the adjacent dataset, Balanced Iteration Subsampling chooses exactly $k$ of the $T$ iterations uniformly at random, and the differing example contributes to those $k$ iterations. Thus, the output distribution is
$$ P = \frac{1}{\binom{T}{k}} \sum_{\substack{S \subseteq [T] \\ |S| = k}} \mc{N}(\mathbf{1}_S,\sigma^2 I_T), $$
where $\mathbf{1}_S \in \{0,1\}^T$ is the indicator vector of the set $S$.

For a specific realization $y = (y_1,\dots,y_T)$, the likelihoods evaluate to
$$ Q(y) = \prod_{i=1}^T \phi_\sigma(y_i), \quad\quad P(y) = \frac{1}{\binom{T}{k}} \sum_{\substack{S \subseteq [T] \\ |S| = k}} \prod_{i \in S} \phi_\sigma(y_i-1) \prod_{i \notin S} \phi_\sigma(y_i), $$
where $\phi_\sigma(\cdot)$ is the probability density function of $\mc{N}(0,\sigma^2)$. Dividing the two and simplifying the Gaussian density ratio yields
$$ \frac{P(y)}{Q(y)} = \frac{1}{\binom{T}{k}} \sum_{\substack{S \subseteq [T] \\ |S| = k}} \prod_{i \in S} \frac{\phi_\sigma(y_i-1)}{\phi_\sigma(y_i)}. $$
Defining a weight $ w_i := \exp\left(\frac{2y_i-1}{2\sigma^2}\right)$ for $i=1,\dots,T$, the log-likelihood ratio simplifies to
\begin{equation}
\label{eq:log_likelihood_w}
\log \frac{P(y)}{Q(y)} = \log \left( \frac{1}{\binom{T}{k}} \sum_{\substack{S \subseteq [T] \\ |S| = k}} \prod_{i \in S} w_i \right).
\end{equation}

The key observation to bypass the $\binom{T}{k}$ combinatorial explosion is that after fixing a realization $y$, the remaining contribution to the likelihood ratio from iteration $t$ onward depends only on two state variables: the number of iterations remaining, and the number of participations still to place. 

For any starting iteration $t \in \{1,\dots,T+1\}$ and remaining participations $r \in \{0,\dots,k\}$, we define the total suffix contribution as:
$$ F(t,r) = \sum_{\substack{S \subseteq \{t,\dots,T\} \\ |S| = r}} \prod_{i \in S} w_i. $$
This yields a clean dynamic programming recursion. For the suffix starting at $t$, either iteration $t$ is not selected (contributing $F(t+1,r)$), or iteration $t$ is selected (contributing $w_t F(t+1,r-1)$). The recursion is therefore:
$$ F(t,r) = F(t+1,r) + w_t F(t+1,r-1), \qquad 1 \le t \le T,\; 1 \le r \le k, $$
with boundary conditions $F(t,0) = 1$ for all $t$ (all participations placed), and $F(T+1,r) = 0$ for all $r > 0$ (no iterations remaining to place participations). 

The exact log-likelihood ratio is then extracted as $\log F(1,k) - \log \binom{T}{k}$. This dynamic program maintains $(T+1)(k+1)$ states, reducing the time complexity of exactly evaluating the likelihood ratio for a single sample to $\mc{O}(Tk)$.

\paragraph{Implementation Note.} We employ appropriate stabilization mechanisms to prevent numerical overflows during the evaluation of the dynamic programming recurrence. Also, the spatial footprint can be reduced to $\mc{O}(k)$ by storing a 1D array of the current state variables and updating it iteratively. While these standard techniques are not explicitly detailed in the text, our complete implementation is provided in the supplementary material for exact reproducibility.

\subsection{A Fast Screening Bound}

While the $\mc{O}(Tk)$ dynamic program is a massive improvement, evaluating it for millions of samples during Monte Carlo verification is still computationally demanding. The real break-through is that to further accelerate verification, we derive an $\mc{O}(T)$ upper bound to act as a computational filter.

\begin{lemma}
\label{lem:screening_bound}
For any nonnegative $u_1,\dots,u_n$ and any $r \in \{0,\dots,n\}$,
$$ A_{n,r}(u_1,\dots,u_n) \coloneqq \frac{1}{\binom{n}{r}} \sum_{\substack{S \subseteq [n] \\ |S| = r}} \prod_{i \in S} u_i \le \left(\frac{1}{n}\sum_{i=1}^n u_i\right)^r. $$
\end{lemma}

\begin{proof}
Fix all coordinates except $u_a$ and $u_b$. Since each term in the sum contains any coordinate at most once, the quantity $A_{n,r}(u_1,\dots,u_n)$ is an affine function of each variable separately. As a function of the pair $(u_a,u_b)$, it takes the form
$$ C_0 + C_1(u_a+u_b) + C_2 u_a u_b, $$
where $C_0,C_1,C_2 \ge 0$ depend only on the remaining coordinates.

Now fix the sum $u_a+u_b=s$. Then the expression becomes
$$ C_0 + C_1 s + C_2 u_a(s-u_a). $$
Because
$$ u_a(s-u_a) \le \frac{s^2}{4} $$
(with equality at $u_a = u_b = s/2$), replacing $(u_a,u_b)$ with their average pair
$$ \left(\frac{u_a+u_b}{2}, \frac{u_a+u_b}{2}\right) $$
cannot decrease $A_{n,r}$.

Applying this averaging operation repeatedly to all pairs, while preserving the total sum $\sum_{i=1}^n u_i$, we eventually converge to the limit at the constant vector $(\bar{u},\dots,\bar{u})$, where
$$ \bar{u} = \frac{1}{n}\sum_{i=1}^n u_i. $$
The function $A_{n,r}$ can only increase along the way, meaning
$$ A_{n,r}(u_1,\dots,u_n) \le A_{n,r}(\bar{u},\dots,\bar{u}). $$
When all coordinates equal $\bar{u}$, every $r$-fold product is $\bar{u}^r$, yielding
$$ A_{n,r}(\bar{u},\dots,\bar{u}) = \bar{u}^r, $$
which completes the proof.
\end{proof}

Applying Lemma \ref{lem:screening_bound} with $u_i = w_i$ yields a fast upper bound on the log-likelihood ratio,
\begin{equation}
\label{eq:fast_bound}
\log \frac{P(y)}{Q(y)} \le k \log \left(\frac{1}{T}\sum_{i=1}^T w_i\right).
\end{equation}
Evaluating \eqref{eq:fast_bound} requires only a single pass over the $T$ iterations, with a computation time of $\mc{O}(T)$.

\subsection{The Screen-Then-Exact Paradigm}

We do not use this $\mc{O}(T)$ upper bound as a loose replacement for the exact likelihood ratio; rather, we deploy it as a computational filter. 

In Monte Carlo accounting, a sample $y$ contributes a non-zero value to the empirical hockey-stick divergence $\hat{\delta}$ only if its exact log-likelihood ratio exceeds the target privacy budget $\varepsilon$. Therefore, if the $\mc{O}(T)$ upper bound in \eqref{eq:fast_bound} evaluates to a value $\le \varepsilon$, it guarantees that the exact log-likelihood ratio is also $\le \varepsilon$. In this case, the exact likelihood ratio is irrelevant to the final sum—we can increment a ``zero-value counter'' and discard the sample, bypassing the dynamic program entirely. The more expensive $\mc{O}(Tk)$ exact evaluation is triggered only when the fast upper bound exceeds $\varepsilon$. 

Empirically, we observe that this screening bound is highly effective. When Monte Carlo accounting is run at a noise level that yields roughly $(\varepsilon,\delta)$-DP, the upper bound only exceeds $\varepsilon$ for approximately an $\mc{O}(\delta)$ fraction of the sampled realizations. Consequently, almost all samples are processed in $\mc{O}(T)$ time, reducing the heuristic expected runtime of the near-exact accountant to $\mc{O}(T + \delta T k)$ per sample, making large-scale privacy verification practical.

\section{Empirical Evaluation}
\label{sec:evaluation}

We now empirically verify the theoretical optimality derived in Section \ref{sec:optimality}. By deploying our near-exact Monte Carlo accountant developed in Section \ref{sec:accounting}, we can strip away the analytical slack of prior methods to observe the true privacy-noise profile of Balanced Iteration Subsampling. 

In Table~\ref{tab:main_results}, we compare the minimum noise multiplier $\sigma$ required to achieve a target $(\varepsilon,\delta)$ privacy guarantee under four different accounting methods:
\begin{itemize}
    \item Privacy loss distribution (PLD) accounting for Poisson subsampling \cite{jax-privacy2022github}, which is numerically exact and serves as the state-of-the-art baseline. This is labeled as ``Poisson''.
    \item The R\'enyi differential privacy (RDP) analysis for Balanced Iteration Subsampling \cite{dong2025leveraging}. This is labeled as ``BIS-RDP''.
    \item The recent PLD-based random allocation accountant, which upper bounds $k$-out-of-$T$ sampling via $k$-fold composition of a $1$-out-of-$\lfloor T/k \rfloor$ allocation \cite{feldman2026efficient, shenfeld2026pld_accounting}. This is labeled as ``RA-PLD''.
    \item Our Monte Carlo accounting for Balanced Iteration Subsampling. This is labeled as ``BIS-MC''.
\end{itemize}

The parameter configurations used in our experiments are drawn from prior empirical studies spanning a diverse range of DP training tasks \cite{de2022unlocking, dong2025leveraging, sander2023tan, li2021large, abadi2016deep}. In real-world applications, the hyperparameters $(\varepsilon, \delta, T, k)$ are not chosen in isolation; rather, they are deeply interdependent, dictated by the holistic dynamics of model convergence, dataset scale, and computational constraints. By adopting explicitly documented training configurations, our evaluation avoids arbitrarily constructed parameter grids that often produce synthetic, unrepresentative edge cases. Instead, it captures the true joint distribution of parameters that practitioners actually navigate, ensuring that the empirical advantages of Balanced Iteration Subsampling are directly realizable in practical, end-to-end training pipelines.

To obtain end-to-end $(\varepsilon, \delta)$-DP guarantees for our Monte Carlo method, we employ a sequential Estimate-Verify-Release (EVR) framework \cite{dong2026privacy}, performing a discrete line search over candidate noise multipliers in decrements of one unit of the third significant digit. We report both a \textit{certified} and an \textit{optimistic} variant of our Monte Carlo accounting. The certified approach yields a rigorous privacy guarantee by absorbing the statistical failure probability into $\delta$. The optimistic procedure instead uses the average of the samples, so it does not satisfy formal end-to-end DP. However, it isolates the statistical slack in finite sampling so serves as an estimate of the ground-truth privacy cost of BIS and represents the fundamental limit that the certified bound converges toward given an increasing amount of compute. It also demonstrates the noise multipliers that would be formally attainable if a closed-form analysis were discovered.

\begin{longtable}{c c c c c c c c c}
\caption{Minimum noise multiplier $\sigma$ required to achieve a target $(\varepsilon,\delta)$ under different accounting methods. Bold entries indicate the lowest noise among formal end-to-end DP methods, including any values within one unit of the third significant digit for our line search precision (a relative difference of $\le$ 1\%). The optimistic Monte Carlo approach lacks a formal DP guarantee, but yields \textit{the lowest noise multiplier across all regimes}.}
\label{tab:main_results}\\
\toprule
$\varepsilon$ & $\delta$ & $T$ & $k$
& Poisson & BIS-RDP & RA-PLD
& \shortstack{BIS-MC \\ {\scriptsize (certified)}}
& \shortstack{BIS-MC \\ {\scriptsize (optimistic)}} \\
\midrule
\endfirsthead

\toprule
$\varepsilon$ & $\delta$ & $T$ & $k$
& Poisson & BIS-RDP & RA-PLD
& \shortstack{BIS-MC \\ {\scriptsize (certified)}}
& \shortstack{BIS-MC \\ {\scriptsize (optimistic)}} \\
\midrule
\endhead

8.0 & $8.33 \cdot 10^{-6}$ & 176 & 3 & 0.560 & 0.588 & 0.521 & \textbf{0.506} & 0.505 \\
8.0 & $1.25 \cdot 10^{-5}$ & 391 & 5 & 0.566 & 0.610 & 0.534 & \textbf{0.521} & 0.521 \\
8.0 & $1.25 \cdot 10^{-5}$ & 195 & 5 & 0.617 & 0.669 & 0.583 & \textbf{0.566} & 0.565 \\
3.0 & $8.33 \cdot 10^{-6}$ & 176 & 3 & 0.787 & 1.02 & 0.750 & \textbf{0.738} & 0.736 \\
8.0 & $1.25 \cdot 10^{-5}$ & 781 & 10 & 0.606 & 0.688 & 0.587 & \textbf{0.577} & 0.576 \\
8.0 & $1.25 \cdot 10^{-5}$ & 391 & 10 & 0.691 & 0.772 & 0.654 & \textbf{0.639} & 0.638 \\
3.0 & $1.25 \cdot 10^{-5}$ & 391 & 5 & 0.772 & 1.06 & 0.748 & \textbf{0.740} & 0.738 \\
3.0 & $1.25 \cdot 10^{-5}$ & 195 & 5 & 0.898 & 1.19 & 0.852 & \textbf{0.838} & 0.837 \\
3.0 & $1.25 \cdot 10^{-5}$ & 6250 & 10 & 0.609 & 0.874 & \textbf{0.597} & \textbf{0.596} & 0.595 \\
3.0 & $1.25 \cdot 10^{-5}$ & 781 & 10 & 0.855 & 1.22 & 0.829 & \textbf{0.822} & 0.820 \\
3.0 & $1.25 \cdot 10^{-5}$ & 391 & 10 & 1.03 & 1.41 & 0.981 & \textbf{0.970} & 0.968 \\
3.0 & $1.25 \cdot 10^{-5}$ & 195 & 10 & 1.30 & 1.69 & 1.22 & \textbf{1.20} & 1.20 \\
8.0 & $1.25 \cdot 10^{-5}$ & 2344 & 30 & 0.717 & 0.860 & 0.702 & \textbf{0.695} & 0.694 \\
8.0 & $1.25 \cdot 10^{-5}$ & 1172 & 30 & 0.840 & 1.00 & 0.817 & \textbf{0.806} & 0.805 \\
8.0 & $1.25 \cdot 10^{-5}$ & 3906 & 50 & 0.803 & 0.974 & 0.780 & \textbf{0.774} & 0.773 \\
8.0 & $1.25 \cdot 10^{-5}$ & 1953 & 50 & 0.961 & 1.15 & 0.933 & \textbf{0.924} & 0.923 \\
1.26 & $10^{-5}$ & 1000 & 10 & 1.22 & 2.09 & \textbf{1.19} & \textbf{1.18} & 1.18 \\
3.0 & $1.25 \cdot 10^{-5}$ & 2344 & 30 & 1.12 & 1.59 & \textbf{1.09} & \textbf{1.09} & 1.08 \\
3.0 & $1.25 \cdot 10^{-5}$ & 1172 & 30 & 1.43 & 1.94 & \textbf{1.40} & \textbf{1.39} & 1.38 \\
1.0 & $10^{-5}$ & 164 & 13 & 3.98 & 5.42 & 4.02 & \textbf{3.86} & 3.85 \\
3.0 & $1.25 \cdot 10^{-5}$ & 3906 & 50 & 1.32 & 1.85 & \textbf{1.30} & \textbf{1.30} & 1.29 \\
3.0 & $1.25 \cdot 10^{-5}$ & 1953 & 50 & 1.74 & 2.31 & \textbf{1.72} & \textbf{1.71} & 1.70 \\
6.0 & $10^{-5}$ & 1843 & 151 & 2.80 & 3.30 & 2.82 & \textbf{2.77} & 2.77 \\
8.0 & $10^{-5}$ & 2468 & 202 & 2.56 & 2.96 & 2.57 & \textbf{2.53} & 2.53 \\
8.0 & $10^{-5}$ & 10000 & 205 & \textbf{1.42} & 1.69 & \textbf{1.42} & \textbf{1.41} & 1.40 \\
8.0 & $10^{-5}$ & 2500 & 205 & 2.58 & 2.98 & 2.59 & \textbf{2.55} & 2.55 \\
8.0 & $8 \cdot 10^{-7}$ & 1000 & 205 & 4.37 & 4.92 & 4.79 & \textbf{4.33} & 4.30 \\
1.5 & $10^{-5}$ & 531 & 43 & 4.94 & 6.43 & 5.05 & \textbf{4.88} & 4.87 \\
7.1 & $2 \cdot 10^{-5}$ & 2500 & 205 & 2.75 & 3.23 & 2.77 & \textbf{2.73} & 2.72 \\
13.7 & $4 \cdot 10^{-5}$ & 2500 & 410 & 3.12 & 3.49 & 3.14 & \textbf{3.10} & 3.09 \\
0.5 & $10^{-5}$ & 781 & 16 & 4.15 & 6.37 & 4.30 & \textbf{4.11} & 4.09 \\
4.0 & $10^{-5}$ & 1687 & 138 & 3.73 & 4.52 & 3.78 & \textbf{3.71} & 3.69 \\
7.3 & $2.5 \cdot 10^{-5}$ & 2500 & 256 & 3.29 & 3.83 & 3.42 & \textbf{3.26} & 3.26 \\
2.0 & $10^{-5}$ & 906 & 74 & 5.00 & 6.36 & 5.11 & \textbf{4.97} & 4.95 \\
3.0 & $10^{-5}$ & 1593 & 130 & 4.63 & 5.70 & 4.71 & \textbf{4.60} & 4.58 \\
4.0 & $10^{-5}$ & 2171 & 178 & 4.23 & 5.10 & 4.29 & \textbf{4.20} & 4.19 \\
2.0 & $10^{-5}$ & 1125 & 92 & 5.56 & 7.05 & 5.70 & \textbf{5.53} & 5.51 \\
4.0 & $8 \cdot 10^{-7}$ & 1000 & 205 & \textbf{7.89} & 9.25 & 8.81 & \textbf{7.88} & 7.83 \\
10.0 & $10^{-5}$ & 2000 & 655 & \textbf{7.37} & 8.37 & 7.49 & \textbf{7.36} & 7.34 \\
1.0 & $10^{-5}$ & 875 & 72 & 9.17 & 12.1 & 9.57 & \textbf{9.14} & 9.11 \\
9.0 & $10^{-5}$ & 2000 & 655 & \textbf{8.02} & 9.14 & 8.17 & \textbf{8.02} & 7.99 \\
1.26 & $10^{-5}$ & 10000 & 100 & \textbf{3.12} & 4.42 & 3.23 & \textbf{3.12} & 3.11 \\
8.0 & $10^{-5}$ & 2000 & 655 & \textbf{8.84} & 10.2 & 9.02 & \textbf{8.83} & 8.80 \\
1.0 & $10^{-5}$ & 250 & 82 & \textbf{19.4} & 25.5 & 20.4 & \textbf{19.5} & 19.4 \\
2.0 & $10^{-5}$ & 500 & 164 & \textbf{14.7} & 18.4 & 15.3 & \textbf{14.7} & 14.7 \\
4.0 & $10^{-5}$ & 1000 & 328 & \textbf{11.3} & 13.5 & 11.6 & \textbf{11.3} & 11.3 \\
7.0 & $10^{-5}$ & 2000 & 655 & \textbf{9.87} & 11.4 & 10.1 & \textbf{9.86} & 9.84 \\
0.5 & $8 \cdot 10^{-7}$ & 250 & 51 & \textbf{26.4} & 34.6 & 31.8 & \textbf{26.5} & 26.3 \\
1.0 & $8 \cdot 10^{-7}$ & 500 & 102 & \textbf{19.6} & 24.7 & 22.8 & \textbf{19.7} & 19.5 \\
4.0 & $8 \cdot 10^{-7}$ & 2000 & 409 & \textbf{11.1} & 13.0 & 12.5 & \textbf{11.1} & 11.1 \\
2.0 & $8 \cdot 10^{-7}$ & 1000 & 205 & \textbf{14.7} & 17.8 & 16.8 & \textbf{14.8} & 14.7 \\
6.0 & $10^{-5}$ & 2000 & 655 & \textbf{11.2} & 13.1 & 11.5 & \textbf{11.3} & 11.2 \\
6.0 & $10^{-5}$ & 2007 & 658 & \textbf{11.3} & 13.2 & 11.6 & \textbf{11.3} & 11.3 \\
5.0 & $10^{-5}$ & 2000 & 655 & \textbf{13.1} & 15.5 & 13.5 & \textbf{13.2} & 13.1 \\
4.0 & $10^{-5}$ & 1765 & 578 & \textbf{14.9} & 17.9 & 15.4 & \textbf{15.0} & 14.9 \\
4.0 & $10^{-5}$ & 2000 & 655 & \textbf{15.9} & 19.0 & 16.5 & \textbf{15.9} & 15.9 \\
3.0 & $10^{-5}$ & 1656 & 543 & \textbf{18.6} & 22.7 & 19.4 & \textbf{18.7} & 18.6 \\
2.0 & $10^{-5}$ & 1156 & 379 & \textbf{22.3} & 27.9 & 23.5 & \textbf{22.4} & 22.3 \\
1.0 & $8 \cdot 10^{-7}$ & 1000 & 205 & \textbf{27.7} & 35.0 & 33.2 & 27.9 & 27.7 \\
3.0 & $10^{-5}$ & 2000 & 655 & \textbf{20.4} & 24.9 & 21.4 & \textbf{20.5} & 20.4 \\
1.0 & $10^{-5}$ & 906 & 297 & \textbf{36.9} & 48.4 & 40.7 & \textbf{37.0} & 36.9 \\

\bottomrule
\end{longtable}

All Monte Carlo privacy accounting evaluations were executed on the Stanford Sherlock computing cluster using general-purpose CPU nodes (primarily AMD Milan architecture). By leveraging the highly parallelizable nature of Monte Carlo sampling, each configuration evaluation was distributed across 16 to 32 CPU cores via the Slurm workload manager, with memory allocations ranging from 24 GB to 32 GB per task.

The required wall-clock execution time varied depending on the target privacy budget $\delta$, the total number of iterations $T$, and the expected participation count $k$. Across the 60+ configurations evaluated in Table \ref{tab:main_results}, the time required to determine the reported noise multiplier $\sigma$ ranged from approximately 12 minutes for the smallest configurations to 12 hours for the most computationally demanding limits. The majority of the accounting tasks completed in under 4 hours. Crucially, no specialized accelerator hardware (GPUs or TPUs) was required for any of the evaluations. We contextualize this computational overhead and discuss its practical feasibility in Section \ref{sec:background_mc_accounting}.

\subsection{Verifying the Optimality of Balanced Iteration Subsampling}

To analyze when BIS outperforms Poisson subsampling, we investigated the relationship between training hyperparameters and the relative noise reduction ($1 - \sigma_{\text{BIS-MC}} / \sigma_{\text{Poisson}}$). This empirical advantage is overwhelmingly driven by the ratio $\varepsilon/k$ (Pearson correlation 0.93), with $T/k$ (correlation 0.49) and $\delta/k$ (correlation 0.81) acting as secondary factors. 

This aligns perfectly with our theory. The quantities $\varepsilon/k$ and $\delta/k$ heuristically represent the privacy budget allocated per sample participation, while $T/k$ captures appearance sparsity. High ratios naturally dictate the low-noise regime where, as established in Section \ref{sec:low_noise_optimality}, eliminating participation variance yields the greatest benefit. 

Table \ref{tab:main_results} illustrates this by sorting descending $\varepsilon/k$, transitioning from low- to high-noise limits. In the low-noise upper portion, eliminating participation variance allows certified BIS-MC to reduce the required noise by up to 9.6\% compared to Poisson. While our theory predicts this relative gain diverges to infinity as $\sigma \to 0$, our 9.6\% maximum reflects practical DP training configurations from prior literature rather than fabricated infinitesimally small noise settings. This concretely demonstrates that deterministically controlling participation variance outweighs maximizing unconstrained randomness over the entire $2^T$ hypercube.

Furthermore, as predicted in Section \ref{sec:high_noise_optimality}, certified BIS-MC noise requirements converge with the Poisson baseline in the high-noise limit ($\sigma \to \infty$). The worst-case 0.9\% noise increase under certified BIS-MC is a trivial statistical artifact of certifying the finite-sample bound with a pragmatic 6--15\% $\delta$-budget allocation, rather than artificially inflating compute to force the gap to zero. With industry-scale compute, sampling could easily match the ground truth to three significant figures, making our optimistic Monte Carlo estimates (a 0.902 to 1.000 ratio against Poisson) highly achievable. Ultimately, the BIS mechanism consistently matches or outperforms Poisson subsampling across all evaluated regimes.

We omit end-to-end model training benchmarks. As noted in Section \ref{sec:optimality}, example participation patterns under a fixed expected count do not observably impact model utility \cite{dong2025leveraging, dong2026privacy}. Consequently, BIS at a lower noise level is functionally equivalent to Poisson subsampling at that same reduced noise level. Because training utility monotonically decreases with the noise multiplier, the $\sigma$ reductions in Table \ref{tab:main_results} directly translate to improved utility for a fixed privacy budget, or a reduced budget for a fixed utility.

\subsection{The Necessity of Near-Exact Accounting}

Finally, these results underscore the necessity of the screen-then-exact Monte Carlo framework developed in Section \ref{sec:accounting}. The PLD-based random allocation accountant (``RA-PLD'') relies on $k$-fold composition, introducing analytical slack. While effective in favorable regimes, its high variance requires up to 20.5\% more noise than standard Poisson analysis in some settings. 

By avoiding composition, our accountant eliminates this slack, revealing that prior BIS underperformance was merely an artifact of loose accounting. With near-exact evaluation, BIS proves to be a highly competitive, advantageous approach for high-utility DP-SGD.

\section{Conclusion}
\label{sec:conclusion}

While Poisson subsampling has long been the DP-SGD default, we prove it is suboptimal under a fixed expected participation constraint. In the low-noise limit ($\sigma \to 0$)---the regime critical for high-utility deployment---optimality requires eliminating sample participation variance rather than maximizing unconstrained randomness. Balanced Iteration Subsampling (BIS) structurally guarantees this, while also matching the optimal full-batch noise-to-batch-size ratio in the high-noise limit ($\sigma \to \infty$). Thus, BIS emerges as the optimal independent-example mechanism at both noise extremes.

To translate this asymptotic theory to finite regimes, we introduced a ``screen-then-exact'' Monte Carlo accountant. By combining an $\mc{O}(Tk)$ dynamic program with an $\mc{O}(T)$ screening filter, it bypasses the analytical looseness of existing RDP and composition-based bounds to extract true, tight privacy costs at scale. 

Our empirical evaluations tightly corroborate these predictions: BIS consistently outperforms Poisson in practical low-noise regimes—yielding up to a 9.6\% reduction in the required noise multiplier—and converges to the Poisson baseline and the analytically established full-batch fundamental limit at high noise.

Ultimately, these insights establish BIS as a robust, principled alternative to Poisson subsampling. By demonstrating that structured participation—not unconstrained randomness—drives optimal privacy amplification, our variance-eliminating framework refines a core DP intuition and provides a rigorous blueprint for pushing the Pareto frontier of private training.

\newpage

\iftrue
\section*{Acknowledgments}
Some of the computing for this project was performed on the Sherlock cluster. We would like to thank Stanford University and the Stanford Research Computing Center for providing computational resources and support that contributed to these research results.
\fi

\bibliographystyle{apalike}
\bibliography{references}

\end{document}